\def\year{2021}\relax
\documentclass[letterpaper]{article} 
\usepackage{aaai21}  
\usepackage{times}  
\usepackage{helvet} 
\usepackage{courier}  
\usepackage[hyphens]{url}  
\usepackage{graphicx} 
\usepackage{natbib}  
\usepackage{caption} 
\usepackage{amsmath}
\usepackage{amssymb}
\usepackage{amsthm}
\usepackage{enumerate}
\usepackage{booktabs}
\usepackage{xcolor}
\usepackage{array}   
\usepackage{stmaryrd}   

\newcolumntype{L}{>{$}l<{$}} 
\newcolumntype{R}{>{$}r<{$}} 

\newcommand{\Omit}[1]{}

\newcommand{\tup}[1]{\langle #1 \rangle}

\newcommand{\set}[1]{\ensuremath{\left\{#1 \right\}}}
\newcommand{\setst}[2]{\ensuremath{\left\{#1 \mid #2 \right\}}}
\newcommand{\abs}[1]{\ensuremath{\left\vert{#1}\right\vert}}
\newcommand{\citeay}[1]{\citeauthor{#1} (\citeyear{#1})}

\newcommand{\psep}{\ensuremath{\,|\,}}  

\newtheorem{definition}{Definition}

\newtheorem{theorem}[definition]{Theorem}
\newtheorem*{theorem*}{Theorem}

\newenvironment{example}{\noindent\textbf{Example.}\,}{\qed}

\renewcommand{\P}{\mathcal{P}}
\newcommand{\Q}{\mathcal{Q}}
\newcommand{\T}{\mathcal{T}}
\newcommand{\F}{\ensuremath{\mathcal{F}}}
\renewcommand{\S}{\mathcal{S}}
\newcommand{\G}{\mathcal{G}}

\newcommand{\Sel}[1]{Select(#1)}
\newcommand{\Good}[1]{Good(#1)}

\newcommand{\msat}{Max-SAT}
\newcommand{\boolcast}[1]{\llbracket #1 \rrbracket}

\newcommand{\smallpar}[1]{{\medskip \noindent \textbf{#1.} }}

\newcommand{\CC}{C\nolinebreak[4]\hspace{-.05em}\raisebox{.2ex}{\small ++}}

\newcommand{\pplus}{\hspace{-.05em}\raisebox{.15ex}{\footnotesize$\uparrow$}}
\newcommand{\mminus}{\hspace{-.05em}\raisebox{.15ex}{\footnotesize$\downarrow$}}

\newcommand{\tsf}{\ensuremath{T({\cal S},{\cal F})}}
\newcommand{\dtwol}{\texttt{D2L}}

\newcommand{\tsfi}[1]{\ensuremath{T_{#1}({\cal S},{\cal F})}}

\newcommand{\mli}[1]{\textit{#1}}

\providecommand{\noopsort}[1]{}

\makeatletter
\def\blfootnote{\gdef\@thefnmark{}\@footnotetext}
\makeatother

\title{Learning General Policies from Small Examples Without Supervision$^*$}

\author{
Guillem Franc\`es,\textsuperscript{\rm 1}
Blai Bonet,\textsuperscript{\rm 1}
Hector Geffner\textsuperscript{\rm 2}\\
}

\affiliations{
  \textsuperscript{\rm 1}\,Universitat Pompeu Fabra, Barcelona, Spain \\
  \textsuperscript{\rm 2}\,ICREA \& Universitat Pompeu Fabra, Barcelona, Spain \\
  guillem.frances@upf.edu, bonetblai@gmail.com, hector.geffner@upf.edu
}

\begin{document}
\allowdisplaybreaks

\maketitle


\begin{abstract}
 Generalized planning is concerned with the computation of general policies that
 solve multiple instances of a planning domain all at once.
 It has been recently shown that these policies can be computed in two steps:
 first, a suitable abstraction in the form of a qualitative numerical planning
 problem (QNP) is learned from sample plans, then the general policies are
 obtained from the learned QNP using a planner. In this work, we introduce an
 alternative approach for computing more expressive general policies which does
 not require sample plans or a QNP planner. The new formulation is very simple
 and can be cast in terms that are more standard in machine learning: a large
 but finite pool of features is defined from the predicates  in the planning
 examples using  a general grammar, and a small subset of features is sought for
 separating ``good'' from ``bad'' state transitions, and goals from non-goals.
 The problems of finding such a ``separating surface'' while labeling
 the transitions as ``good'' or ``bad'' are jointly addressed as a single
 combinatorial optimization problem expressed as a Weighted  \msat{} problem.
 The advantage of looking for the simplest policy in the given feature space
 that solves the given examples, possibly non-optimally, is that many domains
 have no general, compact policies that are optimal.
 The approach yields general policies for a number of benchmark domains.
\end{abstract}%
\blfootnote{
 $^*$This paper extends \cite{frances-et-al-aaai2021} with an appendix providing
    proofs and further details on the methodology and the empirical results.}

\section{Introduction}

Generalized planning is concerned with the computation of general policies or plans   that   solve   multiple  instances of a given planning domain all
at once \cite{srivastava08learning,bonet:icaps2009,hu:generalized,BelleL16,anders:generalized}. For example, a general plan for clearing a  block $x$
in \textbf{any}  instance of Blocksworld involves a loop where the topmost block above $x$ is picked up and placed on the table
until  no such block remains. A general  plan for solving  any Blocksworld instance is also
possible, like one where  misplaced  blocks and those above them are moved  to the table,
and then  to  their targets in order. The key question in generalized planning is how to represent and compute such general plans from the domain
representation.

In one of the most general formulations,  general policies are obtained from an abstract planning model
expressed as a qualitative numerical planning problem or QNP \cite{sid:aaai2011}. A QNP is a   standard STRIPS planning
model extended with non-negative  numerical variables that can be decreased or increased ``qualitatively''; i.e., by uncertain positive
amounts, short of making the variables negative. Unlike standard planning with numerical variables \cite{helmert:numerical},
QNP planning is decidable,  and QNPs  can be compiled in polynomial time into fully observable non-deterministic (FOND)
problems \cite{bonet:qnps}

The main advantage  of the formulation of generalized planning based on QNPs  is that it  applies to  standard relational domains
where the pool of (ground) actions change from instance to instance.
On the other hand, while the planning domain is assumed to be  given,
the QNP abstraction is not, and hence it has to be written by hand  or learned. 
This is the approach of \citeay{bonet:aaai2019} where
generalized plans are obtained by learning  the QNP abstraction from  the domain
representation and  sample plans, and then solving the abstraction  with a QNP  planner.

In this work, we build on this thread but introduce an alternative approach
for computing general policies that is simpler, yet more powerful.
The learning problem is cast as a \textbf{self-supervised classification
problem} where
(1)~a
pool of features is automatically generated from
a general grammar applied to the domain predicates, and
(2)~a
small subset of features is sought for separating ``good'' from ``bad''
state transitions, and goals from non-goals.
The problems of finding the
``separating surface'' while labeling the transitions as ``good'' or ``bad''
are addressed jointly as a single combinatorial optimization task solved with
a Weighted \msat{} solver.
The approach yields general policies for a number of benchmark domains.

The paper is organized as follows. We first review related work and classical planning,
and introduce a new  language for expressing  general policies motivated by the work on QNPs.
We then present  the learning task, the computational approach for solving it, and
the experimental results.

\section{Related Work}

The computation of  general  plans from domain encodings and  sample plans
has been addressed in a number of  works \cite{khardon:action,martin:generalized,fern:generalized,josh:logical}.
Generalized planning has also been formulated as a problem in first-order logic
\cite{srivastava:generalized,sheila:generalized2019}, and general plans over finite horizons have been
derived using first-order regression \cite{boutilier2001symbolic,wang2008first,van2012solving,sanner:practicalMDPs}.
More recently, general policies for planning have been learned from PDDL domains and sample plans
using {deep learning}  \cite{trevizan:dl,sanner:dl,mausam:dl2}. Deep reinforcement learning methods \cite{atari}
have also  been used  to  generate  general policies from
images  {without assuming  prior symbolic  knowledge}
\cite{sid:sokoban,babyAI}, in certain cases accounting for  
objects and relations through the use of suitable architectures \cite{shanahan:review}. 
Our work  is closest to the
works of \citeay{bonet:aaai2019} and \citeay{frances-et-al-ijcai2019}.
The first provides a model-based approach to generalized planning where an abstract QNP model is learned from the domain
representation and sample instances and plans, which is then solved by a QNP planner \cite{bonet:qnps}.
The second learns a generalized value function in an unsupervised manner, under the assumption
that this function is linear.
Model-based approaches have an advantage over  inductive approaches that learn generalized plans;
like  logical approaches, they guarantee  that the resulting policies (conclusions) are correct provided that
the model (set of premises) is correct. The approach developed in this work does not make use of QNPs or planners
but inherits these  formal properties.

\section{Planning}

A (classical) planning instance is a pair  $P\,{=}\,\tup{D,I}$ where
$D$ is a  first-order planning \textbf{domain} and $I$ is an \textbf{instance}.
The  domain $D$ contains a set of predicate symbols and a set of action schemas with preconditions
and effects given by atoms $p(x_1, \ldots, x_k)$, where $p$ is a $k$-ary predicate symbol,
and each $x_i$ is a variable representing one of the arguments of the action schema.
The instance is a tuple $I\,{=}\,\tup{O,Init,Goal}$, where $O$ is a (finite) set of object
names $c_i$, and $Init$ and $Goal$ are  sets  of ground atoms $p(c_1, \ldots, c_k)$,
where $p$ is a $k$-ary predicate symbol.
This is indeed the structure of planning problems as expressed in  PDDL \cite{pddl:book}.

The states associated with  a problem   $P$ are the possible sets of ground atoms,
and the state graph $G(P)$ associated with $P$  has the states of $P$ as nodes,
an initial state $s_0$ that corresponds to the set of atoms in $Init$,
and a set of goal states $s_G$ with all states that include the atoms in $Goal$.
In addition, the graph has a directed edge $(s,s')$ for each state transition that is possible
in $P$, i.e.\ where there is a ground action $a$  whose preconditions hold in $s$
and whose effects transform $s$ into  $s'$. A state trajectory $s_0, \ldots, s_n$
is possible in $P$ if every transition $(s_i,s_{i+1})$ is possible in $P$,
and it is goal-reaching if $s_n$ is a goal state.  An action sequence $a_0, \ldots, a_{n-1}$
that gives rise to a goal-reaching  trajectory, i.e., where transition
$(s_i,s_{i+1})$ is enabled by ground action $a_i$, is called a plan or solution for $P$.

\section{Generalized Planning}

A key question in generalized planning is how to represent general plans or
policies when the different instances to be solved have different sets of
objects and ground actions. One solution is to work with general features (functions)
that have well defined values over any state of any possible domain instance,
and think of general policies $\pi$ as mappings from feature valuations
into \emph{abstract actions} that denote changes in the feature values \cite{bonet:ijcai2018}.
In this work, we build on this intuition but avoid the introduction of
abstract actions \cite{bonet:width}.

\subsection{Policy Language and Semantics}

The \textbf{features} considered are boolean and numerical. The first are denoted
by letters like $p$, and their (true or false) value in a state $s$ is denoted as $p(s)$.
Numerical features $n$ take non-negative integer values, and their value in a state is denoted as
$n(s)$. The complete set of features is
denoted as $\Phi$ and a joint valuation over all the features in $\Phi$
in a state $s$ is denoted as $\phi(s)$, while an arbitrary valuation as $\phi$.
The expression $\boolcast{\phi}$ denotes the boolean counterpart of $\phi$; i.e.,
$\boolcast{\phi}$ gives a truth value to all the atoms $p(s)$ and $n(s)=0$ for features
$p$ and $n$ in $\Phi$,
without providing the exact value of the numerical features $n$ if $n(s) \neq 0$.
The number of possible \textbf{boolean feature valuations} $\boolcast{\phi}$
is equal to $2^{|\Phi|}$, which is a fixed number, as the set of features $\Phi$
does not change across instances.

The possible \textbf{effects} $E$ on the features in $\Phi$ are $p$ and $\neg p$ for boolean features $p$ in $E$,
and $n\mminus$ and $n\pplus$ for numerical features $n$ in $E$. If $\Phi=\{p,q,n,m,r\}$
and $E=\{p,\neg q, n\pplus, m\mminus\}$, the meaning of the effects in $E$ is that
$p$ must become true, $q$ must become false, $n$ must increase its value, and $m$
must decrease it. The features in $\Phi$ that are not mentioned in $E$,
like $r$, keep their values.
A set of effects $E$ can be thought of as a set of constraints on possible state transitions:

\begin{definition}
 Let $\Phi$ be a set of features over a domain $D$, let $(s,s')$ be a state transition
 over an instance $P$ of $D$, and let $E$ be a set of effects over the features in $\Phi$.
 Then the transition $(s,s')$ is \textbf{compatible with} or \textbf{satisfies} $E$
 when 1)~if $p$ ($\neg p$) in $E$, then $p(s')=true$ (resp. $p(s')=false$),
 2)~if $n\mminus$ ($n\pplus$) in $E$, then $n(s) > n(s')$ (resp. $n(s) < n(s'))$, and
 3)~if $p$ and $n$ are not mentioned in $E$, then $p(s)=p(s')$, and $n(s)=n(s')$ respectively.
\end{definition}

\noindent The {form}  of the general policies considered in this work can then be defined as follows:

\begin{definition}
  A \textbf{general policy}  $\pi_\Phi$ is given by a set of \textbf{rules} $C  \mapsto E$
  where $C$ is a set (conjunction)  of $p$ and $n$ literals for  $p$ and $n$ in $\Phi$,
  and $E$ is  an effect expression. 
\end{definition}

The $p$ and $n$-literals are $p$, $\neg p$, $n{=}0$, and $\neg (n{=}0)$, abbreviated as $n{>}0$.
For a reachable state $s$,  the policy $\pi_\Phi$ is a  filter on  the state transitions $(s,s')$ in  $P$:

\begin{definition}
  A  {general policy} $\pi_\Phi$ \textbf{denotes}  a  mapping from state transitions $(s,s')$ over instances $P \in \Q$ into boolean values.
  A transition $(s,s')$ is \textbf{compatible} with $\pi_\Phi$ if  for some policy rule $C \mapsto E$, 
   $C$ is true in $\phi(s)$ and $(s,s')$ satisfies  $E_i$.
\label{def:compatible}
\end{definition}

\noindent As an illustration of these  definitions, we consider a   policy for achieving the
goals $clear(x)$ and an empty gripper in any Blocksworld instance with a block $x$. 

\smallskip
\begin{example}
  Consider the policy $\pi_\Phi$ given by the following two rules for features
  $\Phi{=}\{H,n\}$, where $H$ is true if a block is being held, and $n$ tracks the number of blocks above $x$:
  \begin{equation}
    \{\neg  H, n > 0\} \mapsto \{H, n\mminus\} \ \ ; \ \ \{H, n > 0\} \mapsto \{\neg H\} \ .
  \label{pi:clear}
  \end{equation}
  \noindent The first rule  says that  when the gripper is empty and there are blocks above $x$,
  then any action that decreases $n$ and makes $H$ true should be selected. The   second one   says
  that when the gripper is not empty and there are blocks above $x$,  any action that  makes $H$ false and
  does not affect the count $n$ should be selected  (this  rules out placing the block being held above $x$,
  as this would increase $n$).
\end{example}

\medskip

The conditions under which a general policy solves a class of problems are the following:

\begin{definition}
  A state trajectory $s_0, \ldots, s_n$ is \textbf{compatible}  with policy $\pi_\Phi$ in an  instance $P$ if $s_0$ is the initial state of $P$ and
  each pair $(s_i,s_{i+1})$ is a possible state transition in $P$ compatible with $\pi_\Phi$. The trajectory is \textbf{maximal}
  if $s_n$ is a goal state, there are no state transitions $(s_n,s)$ in $P$ compatible with $\pi_\Phi$,
  or the trajectory is infinite and does not include a goal state.
\end{definition}

\begin{definition}
  A  general policy $\pi_\Phi$ \textbf{solves}  a class $\Q$ of instances over domain $D$
  if in  each instance $P \in \Q$, all maximal state trajectories compatible with $\pi_\Phi$ reach a goal state. 
\end{definition}

The policy expressed by the rules in~\eqref{pi:clear} can be shown to solve  the class $\Q_{clear}$ of all Blocksworld instances.

\subsection{Non-deterministic Policy Rules}

The general policies  $\pi_\Phi$ introduced above  determine the actions $a$ to
be taken in a state $s$ \emph{indirectly}, as the actions $a$ that result
in state transitions $(s,s')$  that are compatible with a policy rule  $C \mapsto E$.
If there is a single rule body $C$ that is true in $s$, for the transition  $(s,s')$
to be compatible with $\pi_\Phi$, $(s,s')$ must satisfy the effect $E$. Yet, it is  possible that the bodies $C_i$ of
many rules $C_i \mapsto E_i$ are true in  $s$,  and then for  $(s,s')$ to be compatible
with $\pi_\Phi$  it suffices if $(s,s')$ satisfies one of the effects $E_i$.

For convenience, we abbreviate sets of rules $C_i \mapsto E_i$, $i=1, \ldots, m$,
that have the same body $C_i=C$, as $C \mapsto E_1\, | \, \cdots\,  |\, E_{m}$,
and refer to the latter as a \textbf{non-deterministic rule}. The non-determinism
is on the effects on the features: one effect $E_i$ may increment a feature $n$,
and another effect $E_j$ may decrease it, or leave it unchanged (if $n$ is
not mentioned in $E_j$). Policies $\pi_\Phi$ where all pairs of rules $C \mapsto E$
and $C' \mapsto E'$ have bodies $C$ and $C'$ that are jointly inconsistent
are said to be \textbf{deterministic}.  Previous formulations that cast
general policies as  mappings from  feature conditions into abstract (QNP) actions
yield policies that are deterministic in this way \cite{bonet:ijcai2018,bonet:aaai2019}. 
Non-deterministic policies, however, are strictly more expressive.

\Omit{
In principle, general policies with deterministic rules can express solutions
to any collection of problems. For this, the  policy $\pi_\Phi$ with the single feature $\Phi=\{V^*\}$ that
represents the optimal cost-to-go to solve the problem, and the single rule $\{V^* > 0\} \mapsto\{ V^*\mminus\}$
that says to move closer to the goal,  solves any solvable problem.
However, we are interested  in learning and using  meaningful features that can be computed efficiently,
in time that is linear in the size of the (ground) problems $P$.
More relevant, though, is that there are domains and sets of of features
that support non-deterministic policies but not deterministic ones.
An example before proceeding  highlights this difference.
}

\smallskip
\begin{example}
  Consider a domain \textbf{Delivery} where a truck has to pick up $m$ packages spread on a grid,
  while taking them, one by one, to a single target cell $t$.
  If we consider the collection of instances with one package only,
  call them  \textbf{Delivery-1}, a general policy $\pi_\Phi$  for them   can be expressed using the
  set of features $\Phi=\{n_p,n_t,C,D\}$, where $n_p$ represents  the distance from the agent to the package
  ($0$ when in the same cell or when holding the package), $n_t$ represents the distance from the agent to the target cell,
  and $C$  and $D$ represent that the package is carried and  delivered respectively. One  may be tempted to write
  the policy $\pi_\Phi$ by means of the four deterministic rules:
  \begin{alignat*}{1}
    & r_1: \, \{\neg C, n_p{>}0\} \mapsto  \{n_p\mminus\}  \ ; \  r_2: \, \{\neg C, n_p{=}0\} \mapsto \{C\}  \\
    & r_3:\, \{C, n_t{>}0\} \mapsto \{n_t\mminus\} \ ;   \ r_4: \, \{C, n_t{=}0\} \mapsto \{\neg C, D\} \, .
  \end{alignat*}
  The rules say ``if away from the package,  get closer'',
  ``if don't have the package but in the same cell, pick it up'',
  ``if carrying the package and away from target, get closer  to target'',
  and ``if carrying the package in target cell, drop the package''.
  This policy, however, does not solve \textbf{Delivery-1}.
  The reason is that transitions $(s,s')$ where the agent gets closer to the package
  satisfy the conditions $\neg C$ and $n_p  >0$ of rule $r_1$ but may fail to satisfy its head $\{n_p\mminus\}$.
  This is because the actions that decrease the distance $n_p$ to the package
  may affect the distance $n_t$ of the agent to the target, contradicting $r_1$,
  which says that $n_t$ does not change. To solve \textbf{Delivery-1} with the same
  features, rule $r_1$ must be changed to the non-deterministic rule:
  \[ r'_1 : \, \{\neg C, n_p > 0\} \, \mapsto \,  \{n_p\mminus , n_t\mminus\} \, | \, \{n_p\mminus , n_t\pplus\} \, | \, \{n_p\mminus\}, \]
  which says indeed  that ``when away from the package, move closer to the package for any possible
  effect on the distance $n_t$ to the target, which may decrease, increase, or stay the same.''
  We often abbreviate rules like $r'_1$ as $\{\neg C, n_p > 0\} \mapsto  \{n_p\mminus, n_t?\}$, where  $n_t?$ expresses ``any effect on $n_t$.''
\end{example}

\section{Learning General Policies: Formulation}

We turn now to the key challenge:  \textbf{learning} the features $\Phi$ and general  policies  $\pi_\Phi$
from \textbf{samples} $P_1, \ldots, P_k$ of  a  target class of problems  $\Q$, given the  domain $D$.
The {learning task} is formulated as follows. From the predicates used in $D$
and a fixed grammar, we generate a \textbf{large pool $\F$ of boolean and
numerical features} $f$, like in \cite{bonet:aaai2019}, each of which is
associated with a measure $w(f)$ of syntactic complexity.
We then \emph{search for the simplest set of features
$\Phi \subseteq \mathcal{F}$ such that a policy $\pi_\Phi$ defined on
$\Phi$  solves all sample instances $P_1, \ldots, P_k$.}
This task is formulated as a Weighted \msat{} problem over a
suitable propositional theory $T$, with score $\sum_{f \in \Phi} w(f)$
to minimize.

This learning scheme is \textbf{unsupervised} as the  sample instances do not come with their plans.
Since the sample instances are assumed to be sufficiently small (small state spaces)
this is not a crucial issue, and by  letting the learning algorithm choose which plans
to generalize, the resulting approach becomes more flexible. In particular,
if we ask for the policy $\pi_\Phi$ to generalize  given plans as in \cite{bonet:aaai2019},
it may well happen that there are policies in the feature space but none of which
generalizes the plans provided by the teacher.

We next describe the propositional theory $T$
assuming that the feature pool $\F$ and the feature weights $w(f)$ are given,
and then explain how they are generated.
Our SAT formulation is  different from \cite{bonet:aaai2019} 
as it is aimed at capturing a more expressive class of  policies  without
requiring QNP planners.

\subsection{Learning the General Policy as Weighted  \msat{}}

The propositional theory $T=T(\S,\F)$ that captures our learning task
takes as inputs the pool of features $\F$ and the state space $\S$ made up of the (reachable) states $s$,
the possible state transitions $(s,s')$,
and the sets of (reachable) goal states in each of the sample problem instances
$P_1, \ldots, P_n$.
The handling of dead-end states is explained below.
States arising from the different instances are assumed to be different even
if they express the same set of ground atoms. 
The propositional variables in $T$ are
\begin{itemize}
  \item $Select(f)$: feature $f$ from pool $\F$ makes it into $\Phi$,
  \item $Good(s,s')$: transition $(s,s')$ is compatible with $\pi_\Phi$,
  \item $V(s,d)$: num. labels $V(s)=d$, $V^*(s) \leq d \leq \delta  V^*(s)$.
\end{itemize}

The true atoms $Select(f)$ in the satisfying assignment define the features $f \in \Phi$,
while the true atoms $Good(s,s')$, along with the selected features, define the policy $\pi_\Phi$.
More precisely, there is a rule $C \mapsto E_1 \, | \, \cdots \, | E_m$ in the policy iff
for each effect $E_i$, there is a true atom $Good(s,s_i)$ for which $C = \boolcast{\phi(s)}$,
and $E_i$ captures the way in which the selected features change across  the transition $(s,s_i)$.
The formulas in the theory use numerical labels $V(s)=d$, for $V^*(s) \leq d \leq \delta  V^*(s)$
where $V^*(s)$ is the minimum distance from $s$ to a goal, and $\delta \geq 1$ is a \emph{slack parameter} that controls
the degree of suboptimality that we allow. All experiments in this paper use $\delta=2$.
These values are used to ensure that the policy determined by the $Good(s,s')$ atoms
solves all instances $P_i$ as well as all instances $P_i[s]$ that are like $P_i$
but with $s$ as the initial state, where $s$ is a state reachable in $P_i$
and is not a dead-end.
We call the $P_i[s]$ problems \textbf{variants} of $P_i$.
Dead-ends are states from which the goal cannot be reached,
and they are labeled as such in $\S$.

The formulas are the following.
States $s$ and $t$, and transitions $(s,s')$ and $(t,t')$
range over those in $\S$, excluding transitions where the first state
of the transition is a dead-end or a goal.
$\Delta_f(s,s')$ expresses  how feature $f$ changes across transition $(s,s')$:
for boolean features, $\Delta_f(s,s') \in \{\uparrow, \downarrow, \bot \}$, meaning that $f$ changes
from false to true,
from true to false,
or stays the same. For numerical features,
$\Delta_f(s,s')  \in \{\uparrow, \downarrow\, \bot\}$, meaning that $f$ can increase,
decrease, or stay the same.
The formulas in $T=T(\S,\F)$ are:

\begin{enumerate}[1.]
  \item Policy: $\bigvee_{(s,s')} Good(s,s')$, $s$ is non-goal state,
  \item $V_1$.: $\text{Exactly-1}\, \{V(s,d): V^*(s) \leq d \leq \delta V^*(s)\}$,\footnote{
This implies that $V(s,0)$ iff $s$ is a goal state.
}
  \item $V_2$: $Good(s,s') \rightarrow V(s,d) \land V(s',d')$,  $d' < d$,
  \item Goal: $\bigvee_{f: \boolcast{f(s)} \not= \boolcast{f(s')}} \  Select(f)$,  one $\{s,s'\}$ is goal,
  \item Bad trans: $\neg Good(s,s')$ for $s$ solvable, and $s'$ dead-end,
  \item D2-sep: $Good(s, s') \land \neg Good(t, t') \rightarrow D2(s,s'; t,t')$, where $D2(s,s';t,t')$ is $\bigvee_{\Delta_f(s,s') \not= \Delta_f(t,t')}  Select(f)$.
\end{enumerate}

The first formula asks for a good transition from any non-goal state $s$.
The good transitions are transitions that will be compatible with the policy.
The second and third formulas ensure that these good transitions lead to a goal state, and
furthermore, that they can capture any non-deterministic policy that does so. 
The fourth formulation is about separating goal from non-goal states,
and the fifth is about excluding transitions into dead-ends.
Finally, the D2-separation formula says that if $(s,s')$ is a ``good'' transition (i.e., compatible with
the resulting policy $\pi_\Phi$), then any other transition $(t,t')$ in
$\cal S$ where the selected features change exactly  as in $(s,s')$
must be ``good'' as well. $\Delta_f(s,s')$ above captures
how  feature $f$ changes across the transition $(s,s')$,
and the selected features $f$ change in the same way in $(s,s')$
and $(t,t')$ when $\Delta_f(s,s')=\Delta_f(t,t')$.

The propositional encoding is \textbf{sound} and \textbf{complete} in the following sense:


\begin{theorem}
Let $\S$ be the state space associated with a set $P_1, \ldots, P_k$ of
sample instances of a class of problems $\Q$ over a domain $D$,
and let $\F$ be a pool of features.
The theory $\tsf$ is \textbf{satisfiable} iff
there is a general policy $\pi_\Phi$ over features $\Phi \subseteq \F$
that discriminates goals from non-goals and solves $P_1$, \ldots, $P_k$ and
their variants.
\label{thm:sat}
\end{theorem}

For the purpose of generalization outside of the sample instances, instead  of looking for \textbf{any} satisfying
assignment of the theory $T({\cal S},{\cal F})$, we   look for the  satisfying  assignments that \textbf{minimize}
the complexity of the resulting policy, as measured by the sum of the costs $w(f)$ of the clauses $Select(f)$ that
are true,  where $w(f)$ is the complexity of feature $f \in {\cal F}$.

We sketched above how a  general policy $\pi_\Phi$ is extracted from  a satisfying assignment.
The only thing missing is the   precise meaning of the line  ``$E_i$ captures the way in which
the selected features change in the transition from $s$ to $s_i$''.
For this, we look at the value of the expression  $\Delta_f(s,s_i)$
computed at preprocessing, and place $f$ ($\neg f$) in $E_i$ if $f$ is boolean and $\Delta_f(s,s_i)$ is `$\uparrow$' (resp. $\downarrow$),
and place $f\pplus$ ($f\mminus$) in $E_i$ if $f$ is numerical and $\Delta_f(s,s_i)$ is `$\uparrow$' (resp. $\downarrow$).
Duplicate effects $E_i$ and $E_j$ in a policy rule are merged. The resulting policy
delivers the properties of Theorem~\ref{thm:sat}:

\begin{theorem}
  The policy $\pi_\Phi$ and features $\Phi$ that  are determined by  a satisfying
  assignment of the theory $T$
  solves the sample problems $P_1, \ldots, P_k$ and their variants.
\label{thm:several-initial-states}
\end{theorem}

\section{Feature Pool}

The feature pool $\F$ used in the theory  $T(\S,\F)$
is obtained following the method described by \citeay{bonet:aaai2019},
where the (primitive) domain predicates are combined through a standard description logics
grammar~\cite{dl-handbook} in order to build a larger set of (unary) concepts $c$ and (binary) roles $r$.
Concepts represent \emph{properties} that the objects of any problem instance can fulfill in a state,
such as the property of being a package that is in a truck on its target location in a standard logistics problem.
For primitive predicates $p$ mentioned in the goal,
a ``goal predicate'' $p_G$ is added that is evaluated
not in the state but in the goal, following \cite{martin:generalized}.

From these concepts and roles, we generate \emph{cardinality features} $\abs{c}$,
which evaluate to the number of objects that satisfy concept $c$ in a given state,
and \emph{distance features} $Distance(c_1, r, c_2)$, which evaluate to the minimum number of $r$-steps between
two objects that (respectively) satisfy $c_1$ and  $c_2$.
We refer the reader to the appendix for more detail. 
Both types of features are lower-bounded by $0$ and upper-bounded by the total number of objects in the problem instance.
Cardinality features that  only take values in $\set{0,1}$ are made into boolean features.
The complexity $w(f)$ of feature $f$ is given by the size of 
its syntax tree. The feature pool $\F$ used in the experiments below contains all features up to a certain complexity
bound $k_\F$.

\section{Experimental Results}
We implemented the proposed approach in  a \CC/Python system called \dtwol{}
and evaluated it on several problems.
Source code and benchmarks are available online\footnote{
\url{https://github.com/rleap-project/d2l}.
}
and archived in Zenodo \cite{d2l2021zenodo}.
Our implementation uses the Open-WBO Weighted \msat{} solver \cite{martins2014open}.
All experiments were run on an Intel i7-8700 CPU@3.2GHz with a 16 GB memory  limit.

The domains include all problems with simple goals from~\cite{bonet:aaai2019},
e.g.\ clearing a block or stacking two blocks in Blocksworld, plus standard PDDL domains
such as Gripper, Spanner, Miconic, Visitall and Blocksworld.
In all the experiments, we use $\delta=2$ and $k_\F=8$, except in Delivery, where
$k_\F=9$ is required to find a policy.
We next describe two important optimizations.

\smallpar{Exploiting indistinguishability of constraints}
A fixed feature pool $\F$ induces an equivalence relation over the set of all transitions in the training sample
that puts two transitions in the same equivalence class iff they cannot be distinguished by $\F$.
The theory $\tsf$ above can be simplified by arbitrarily choosing one transition $(s,s')$ for each of these
equivalence classes, then using a single SAT variable $Good(s,s')$ to denote the goodness of any transition in the class
and to enforce the D2-separation clauses.

\smallpar{Incremental constraint generation}
Since the number of D2-separation
constraints in the theory \tsf{} grows quadratically with the number of
equivalence classes among the transitions, we use a \emph{constraint generation
loop} where these constraints are enforced incrementally.
We start with a set $\tau_0$ of pairs of transitions $(s,s')$ and $(t,t')$ that
contains all pairs for which $s=t$ plus some random pairs from $\S$.
We obtain the theory $\tsfi{0}$ that is like $\tsf$ but where the D2-separation
constraints are restricted to pairs in $\tau_0$.
At each step, we solve $\tsfi{i}$ and validate the solution to check whether
it distinguishes all good from bad transitions in the entire sample;
if it does not, the offending transitions are added to
$\tau_{i+1} \supset \tau_i$, and the loop continues until
the solution to $\tsfi{i}$ satisfies the D2-separation formulas
for all pairs of transitions in $\S$, not just those in $\tau_i$.

\subsection{Results}

Table~\ref{tab:exp-result-data} provides an overview of the execution of
\dtwol{} over all generalized domains.
The two main conclusions to be drawn from the results are that
1)~our generalized policies are  more expressive and
result in policies that cannot be captured in previous approaches \cite{bonet:aaai2019},
2)~our SAT encoding is also simpler and scales up much better,
allowing to tackle harder tasks with reasonable computational effort.
Also, the new  formulation is unsupervised and complete, in the sense that if there is a general policy in the given
feature space that solves the instances, the solver is guaranteed to find it.
%

In all domains, we use a modified version of the Pyperplan planner\footnote{\url{https://github.com/aibasel/pyperplan}.}
to check  empirically that the learned policies are able to solve a set of test instances of
significantly larger dimensions than the training instances. For standard PDDL domains with readily-available instances (e.g., Gripper, Spanner, Miconic),
the test set  includes all instances in the benchmark set,\footnote{
We have used the benchmark distribution in \url{https://github.com/aibasel/downward-benchmarks}.
} whereas for other domains such as $\Q_{rew}$, $\Q_{deliv}$ or  $\Q_{bw}$, the test set contains at least
$30$ randomly-generated instances.

We next briefly describe the policy learnt by \dtwol{} in each domain;
the appendix contains detailed descriptions and proofs of correctness
for all these policies. 

\begin{table*}[t]
  \newcommand{\ninst}{\ensuremath{\abs{P_i}}} 
  \newcommand{\tset}{\ensuremath{\S}} 
  \newcommand{\tseteq}{\ensuremath{\S/{\scriptstyle\sim}}} 
  \newcommand{\tall}{\ensuremath{t_{\text{all}}}} 
  \newcommand{\tsat}{\ensuremath{t_{\text{SAT}}}} 
  \newcommand{\npool}{\ensuremath{\abs{\F}}} 
  \newcommand{\nfeats}{\ensuremath{\abs{\Phi}}} 
  \newcommand{\sizepi}{\ensuremath{\abs{\pi_\Phi}}} 
  \newcommand{\cphi}{\ensuremath{c_{\Phi}}} 
    \newcommand{\maxk}{\ensuremath{k^*}} 
  \newcommand{\dmax}{\ensuremath{d_{max}}} 
  \newcommand{\dimens}{\ensuremath{dim}} 
  \newcommand{\TK}{\text{K}} 
  \newcommand{\TM}{\text{M}} 

  \centering
  \resizebox{\textwidth}{!}{
    \begin{tabular}{LRRRRRRRRRRRRRRR}
      \toprule
                  &\ninst&  \dimens &&  \tset & \tseteq &\dmax& \npool  &  vars     & clauses         & \tall &\tsat& \cphi &\nfeats & \maxk & \sizepi \\
      \midrule
      \Q_{clear}  &    1 &        5 &&  1,161 &    55 &  7 &    532  &    7.9\TK &  243.7\TK (242.3\TK) &   6 &<1   &      8 &      3 &   4 &    3 \\
      \Q_{on}     &    1 &        5 &&  1,852 &   329 & 10 &   1,412 &   17.3\TK &  376.6\TK (281.5\TK) & 33 & 22   &     13 &      5 &   5 &    7 \\
      \Q_{grip}   &    1 &        4 &&  1,140 &    61 & 12 &    835  &    6.5\TK &  102.6\TK (100.8\TK) &   2 &<1   &      9 &      3 &   4 &    4 \\
      \Q_{rew}    &    1 & 5\times5 &&    432 &   361 & 15 &    514  &    5.5\TK &  214.9\TK  (98.9\TK) &   2 &<1   &      7 &      2 &   6 &    2 \\
      \Q_{deliv}  &    2 & 4\times4 &&  42,473&  5442 & 56 &   1,373 &  753.4\TK &  38.2 \TM  (23.5\TM) & 3071& 2902&     30 &      4 &  14 &    6 \\
      \Q_{visit}  &    1 & 3\times3 &&  2,396 &   310 &  8 &    188  &   13.9\TK &  244.5\TK (160.6\TK) &   3 &<1   &      7 &      2 &   5 &    1 \\
      \Q_{span}   &    3 & (6, 10)  && 10,777 &    96 & 19 &    764  &   85.0\TK &    2.2\TM   (2.2\TM) &  32 &<1   &      9 &      3 &   6 &    2 \\
      \Q_{micon}  &    2 & (4, 7)   &&  4,706 & 4,636 & 14 &  1,073  &   23.8\TK &   23.6\TM   (2.4\TM) & 41  & 61  &     11 &      4 &   5 &    5 \\
      \Q_{bw}     &    2 &        5 &&  4,275 & 4,275 &  8 &  1,896  &   22.1\TK &   9.3\TM  (390.0\TK) & 80  & 40  &     11 &      3 &   6 &    1 \\
      \bottomrule
    \end{tabular}
  }
  \caption{\emph{Overview of results}.
    \ninst{} is number of training instances, and
    \dimens{} is size of largest training instance along main generalization dimension(s): number of blocks
    ($\Q_{clear}, \Q_{on}, \Q_{bw}$),
    number of balls ($\Q_{grip}$),
    grid size ($\Q_{rew},\Q_{deliv}, \Q_{visit}$),
    number of locations and spanners ($\Q_{span}$),
    number of passengers and floors ($\Q_{micon}$).
    We fix $\delta=2$ and $k_\F=8$ in all experiments except $\Q_{deliv}$, where
    $k_\F=9$.
    \tset{} is number of transitions in the training set, and
    \tseteq{} is the number of distinguishable equivalence classes in \tset{}.
    \dmax{} is the max.\ diameter of the training instances.
    \npool{} is size of feature pool.
    ``Vars'' and ``clauses'' are the number of variables and clauses in the (CNF form) of the theory $\tsf$;
    the number in parenthesis is the number of clauses in the last iteration
    of the constraint generation loop.
    \tall{} is total CPU time, in sec., while
    \tsat{} is CPU time spent solving  \msat{} problems.
    \cphi{} is optimal cost of SAT solution,
    \nfeats{} is number of selected features,
    \maxk{} is cost of the most complex feature in the policy,
    \sizepi{}  is number of rules in the resulting policy.
    CPU times are given for the incremental constraint generation approach.
  }
  \label{tab:exp-result-data}
\end{table*}

\smallpar{Clearing a block}
$\Q_{clear}$ is a simplified Blocksworld where the goal is to get $clear(x)$ for a distinguished block $x$.
We use the standard $4$-op encoding with stack and unstack actions.
Any 5-block training instance suffices to compute the following policy over features
$\Phi = \{c, H, n\}$
that denote, respectively, whether $x$ is clear, whether
the gripper holds a block, and the number of blocks above $x$:\footnote{All features
  discussed in this section are automatically derived with the description-logic
  grammar, but we label them manually for readability.}
\begin{alignat*}{1}
&r_1:\, \set{\neg c, H, n=0}        \ \mapsto\ \set{c, \neg H}\,, \\
&r_2:\, \set{\neg c, \neg H, n>0}   \ \mapsto\ \set{c?, H, n\mminus}\,, \\
&r_3:\, \set{\neg c, H, n>0}        \ \mapsto\ \set{\neg H}\,.
\end{alignat*}
Rule $r_1$ applies only when $x$ is held (the only case where $n=0$ and $\neg c$),
and puts $x$ on the table. Rule $r_2$ picks any block above $x$ that can be picked, potentially
making $x$ clear, and $r_3$ puts down block $y \neq x$ anywhere \emph{not} above $x$.
Note that this policy is slightly more complex than the one defined in (\ref{pi:clear}) because the SAT theory
enforces that goals be distinguishable from non-goals, which in the standard encoding cannot be achieved with
$H$ and $n$ alone.

\smallpar{Stacking two blocks}
$\Q_{on}$ is another simplification of Blocksworld where the goal is $on(x,y)$
for two designated blocks $x$ and $y$.
One training instance with 5 blocks yields a policy over features
$\Phi = \{e, c(x), on(y), ok, c\}$. The first four are boolean and encode
whether the gripper is empty, $x$ is clear, some block is on $y$, and $x$ is on $y$;
the last is numerical and encodes the number of clear objects.
This version of the problem is more general than that in \cite{bonet:aaai2019}, where $x$ and $y$ are assumed
to be initially in different towers.


\smallpar{Gripper}
$\Q_{grip}$ is the standard Gripper domain where a two-arm robot
has to move $n$ balls between two rooms $A$ and $B$.
Any 4-ball instance
is sufficient to learn a simple policy with
features
$\Phi = \{r_B, c, b\}$
that denote whether the robot is at $B$,
the number of balls carried by the robot,
and the number of balls not yet left in $B$:
\begin{alignat*}{1}
&r_1:\, \{\neg r_B, c=0, b>0\}\ \mapsto\ \{c\pplus\}\,, \\
&r_2:\, \{r_B, c=0, b>0\}\ \mapsto\ \{\neg r_B\}\,, \\
&r_3:\, \{r_B, c>0, b>0\}\ \mapsto\ \{c\mminus, b\mminus\}\,, \\
&r_4:\, \{\neg r_B, c>0, b>0\}\ \mapsto\ \{r_B\}\,.
\end{alignat*}
In any non-goal state, the policy is compatible with the transition induced
by some action; overall, it implements a loop that moves balls from A to B,
one by one.
\citeay{bonet:aaai2019} also learn an abstraction for Gripper, but need an extra feature $g$ that counts
the number of free grippers in order to keep the soundness of their QNP model.
Our approach does not need to build such a model, and the policies it learns often
use features of smaller complexity.

\smallpar{Picking rewards}
$\Q_{rew}$ consists on an agent that navigates a grid with some non-walkable cells in order to pick up scattered reward items.
Training on a single $5\times5$ grid with randomly-placed rewards and non-walkable cells results in the same policy
as reported by \citeay{bonet:aaai2019}, which moves the agent to the closest
unpicked reward, picks it, and repeats. In contrast with that work, however,
our approach does not require sample plans, and its propositional theory is one order of magnitude smaller.

\smallpar{Delivery}
$\Q_{deliv}$ is the previously discussed Delivery problem, where a truck needs to pick $m$
packages from different locations in a grid and deliver them, one at a time,
to a single target cell $t$. The policy learnt by \dtwol{} is  a generalization
to $m$ packages of the one-package policy discussed before.

\smallpar{Visitall}
$\Q_{visit}$ is the standard Visitall domain where an agent has to visit all the cells in a grid at least once.
Training on a single $3\times3$ instance produces a single-rule policy based on features $\Phi = \{u, d\}$ that represent
the number of unvisited cells and the distance to a closest unvisited cell. The policy, similar to the one for
$\Q_{rew}$, moves the agent greedily to a closest unvisited until all cells have been visited.


\smallpar{Spanner}
$\Q_{span}$ is the standard Spanner domain where an agent picks up spanners along a corridor
that are used at the end to tighten some nuts.
Since spanners can be used only once and the corridor is one-way, the problem becomes unsolvable
as soon as the agent moves forward and leaves some needed Spanner behind.
We feed \dtwol{} with 3 training instances with different initial locations
of spanners, and it computes a policy with features $\Phi = \{n, h, e\}$ that denote
the number of nuts that still have to be tightened, the number of objects not held by the agent and whether
the agent location is empty, i.e.\ has no spanner or nut in it:
%
\begin{alignat*}{1}
&r_1:\, \set{n>0,  h>0, e}      \ \mapsto\ \set{e?}\,, \\
&r_2:\, \set{n>0,  h>0, \neg e} \ \mapsto\ \set{h\mminus, e?} \psep \set{n\mminus}\,.
\end{alignat*}

The policy dictates a move when the agent is in an empty location; else,
it dictates either to pick up a spanner  or tighten a nut. 
Importantly, it never allows the agent to leave a location with some unpicked spanner,
thereby avoiding dead-ends.
%
%
Note that the features and policy are fit to the domain actions. For instance,
an effect $\set{e?}$ as in $r_1$ could not appear if the domain had \emph{no-op}
actions, as the resulting \emph{no-op} transitions would comply with $r_1$
without making progress to the goal.
The learned policy solves the 30 instances of the learning track of the 2011
International Planning Competition, and can actually be formally proven
correct over all Miconic instances.

\smallpar{Miconic}
$\Q_{micon}$ is the domain where a single elevator moves across different floors to
pick up and deliver passengers to their destinations.
We train on two instances with a few floors and passengers with different origins and destinations.
The learned policy uses 4 numerical features that encode
the number of passengers onboard in the lift,
the number of passengers waiting to board,
the number of passengers waiting to board on the same floor where the lift is, and
the number of passengers boarded when the lift is on their target floor.
The policy solves the 50 instances of the standard Miconic distribution.

\smallpar{Blocksworld}
$\Q_{bw}$ is the classical Blocksworld where the goal is to achieve some desired
arbitrary configuration of blocks, under the assumption that each block
has a goal destination (i.e., the goal picks a single goal state). 
We use a standard PDDL encoding where blocks are moved atomically from one
location to another (no gripper).
The only predicates are $on$ and $clear$, and the set of objects consists of $n$
blocks and the table, which is always clear.
We use a single training instance with $5$ blocks, where the \emph{target}
location of all blocks is specified.
We obtain a policy over the features $\Phi = \{c, t', bwp\}$ that stand for the number of clear objects,
the number of objects that are not \emph{on} their target location, and
the number of objects such that all objects below are well-placed,
i.e., in their goal configuration.
Interestingly, the value of all features in non-goal states is always positive
($bwp>0$ holds trivially, as the table is always well-placed and below all blocks).
The computed policy  has one single rule with four effects:
\begin{align*}
\set{c>0, t'>0, bwp>0}\ \mapsto\
&\set{c \pplus}                                  \psep
\set{c \pplus,  t' ?, bwp \pplus} \psep \\
&\set{c \pplus,  t' \mminus} \psep
\set{c \mminus, t' \mminus}.
\end{align*}

The last effect in the rule is compatible with any move of a block from the table into its final position,
where everything below is already well-placed (this is the only move away from the table compatible with the policy),
while the remaining effects are compatible with moving into the table a block that is not on its final position.
The policy solves a set of 100 test instances with 10 to 30 blocks and random initial and goal configurations,
and can actually be proven correct.

\paragraph{Discussion of Results.}

On dead-end free domains where all instances of the same size (same objects) have isomorphic state spaces,
\dtwol{} is able to generate valid policies from one single training instance.
In these cases, the only choice we have made regarding the training instance is selecting a size for the instance
which is sufficiently large to avoid \emph{overfitting}, but sufficiently small to allow the expansion of the entire
state space. As we have seen, though, the approach is also able to handle domains with dead-ends ($\Q_{span}$)
or where different instances with the same objects can give rise to non-isomorphic state spaces
($\Q_{rew}, \Q_{micon}$). In these cases, the selection of training instances needs to be done more carefully
so that sufficiently diverse situations are exemplified in the training set.

As it can be seen in Table~\ref{tab:exp-result-data}, the two optimizations discussed at the beginning 
are key to scale up in different domains.
Considering indistinguishable classes of transitions instead of individual transitions offers a dramatic reduction
in the size of the theory \tsf{} for domains with a large number of symmetries such as Spanner, Visitall, and Gripper.
On the other hand, the incremental constraint generation loop also reduces the size of the theory up to
one order of magnitude for domains such as Miconic and Blocksworld.

Overall, the size of the propositional theory, which is the main bottleneck in \cite{bonet:aaai2019}, is much smaller. 
Where they report a number of clauses for $\Q_{clear}$, $\Q_{on}$, $\Q_{grip}$ and $\Q_{rew}$ of, respectively,
767K, 3.3M, 358K and 1.2M,
the number of clauses in our encoding is
242.3K, 281.5K, 100.8K and 98.9K, that is up to one order of magnitude smaller, which allows \dtwol{}
to scale up to several other domains. Our approach is also more efficient than the one in \cite{frances-et-al-ijcai2019}.
which requires several hours to solve a domain such as Gripper.

\Omit{
Overall, \dtwol{} is able to learn general policies for domains which seem
beyond the scope of earlier approaches like Visitall, Miconic or Blocksworld,
in some cases due to the better scalability of the encoding, in other
due to the additional expressivity afforded by the new policy space.
}

\section{Conclusions}

We have introduced a new method for learning features and general policies from small problems without supervision.
This is achieved by means of a novel formulation in which a large but finite pool of features is defined from the
predicates in the planning examples using a general grammar, and a small subset of features is sought for separating
``good'' from  ``bad'' state transitions, and goals from non-goals.
The problems  of finding such a ``separating surface'' while labeling
the transitions as ``good'' or ``bad'' are addressed jointly as a Weighted  \msat{} problem.
The formulation is complete in the sense that if there is a general policy  with features in the pool that solves
the training instances, the solver will find it, and by computing the simplest such solution, it ensures a better
generalization outside of the training set. In comparison with existing approaches, the new formulation is
conceptually simpler, more scalable (much smaller propositional theories), and more expressive (richer class of
non-deterministic policies, and value functions that are not necessarily linear in the features).
In the future, we want to  study extensions for synthesizing provable correct policies
exploiting related results in  QNPs.

\subsection*{Acknowledgements}

This research is partially funded by an ERC Advanced Grant (No 885107),
by grant TIN-2015-67959-P from MINECO, Spain,
and by the Knut and Alice Wallenberg (KAW) Foundation through the WASP program.
H.\ Geffner is also a Wallenberg Guest Professor at Link\"oping University, Sweden.
G. Franc\`{e}s is partially supported by grant IJC2019-039276-I from MICINN, Spain.

\bibliography{control}

\appendix

\section{Appendix}

This appendix contains
(1)~proofs for Theorems 6 and 7,
(2)~a detailed description of the feature grammar used by \dtwol{},
(3)~a full account of the generalized policies learned by the approach, and
(4)~proof of their generalization to all instances of the generalized problem.

\section{Theorems 6 and 7}

For a sample $\S$ made of problems $P_1,\ldots,P_k$, let $\P_\S$ be the collection
of problems consisting of $P_i[s]$, for $1\leq i\leq k$ and non-dead-end state $s$ in $P_i$,
where the problem $P_i[s]$ is like problem $P_i$ but with initial state set to $s$.


Theorems 6 and 7 in paper are subsumed by the following:

\begin{theorem*}
    Let $\S$ be the state space associated with a set $P_1, \ldots, P_k$ of sample instances of a class of problems $\Q$ over a domain $D$, and let $\F$ be a pool of features.
    The theory $T(\S,\F)$ is \textbf{satisfiable} iff there is a policy $\pi_\Phi$ over features $\Phi \subseteq \F$
    that solves all the problems in $\P_\S$ and where the features in $\Phi$ discriminate goal from non-goal states in the problems $P_1,\ldots,P_k$.
\end{theorem*}
\begin{proof}
    We need to show two implications. Let us denote the theory $T(\S,\F)$ as $T$.
    For the first direction, let $\pi$ be a policy defined over a subset $\Phi$ of features from $\F$ such that $\pi$ solves any problem in $\P_\S$ and $\Phi$ discriminates the goals in $P_1,\ldots,P_k$.
    Let us construct an assignment $\sigma$ for the variables in $T$:
\begin{list}{--}{}
        \item $\sigma\vDash\Sel{f}$ iff $f\in\Phi$,
        \item $\sigma\vDash\Good{s,s'}$ iff transition $(s,s')$ is compatible with $\pi$,
        \item $\sigma\vDash V(s,d)$ iff $V_\pi(s)=d$, where $V_\pi(s)$ is the \textbf{distance}
        of the max-length path connecting $s$ to a goal
        in the subgraph $\S_\pi$ of $\S$ spanned by $\pi$:
        edge $(s,s')$ belongs to $\S_\pi$ iff $(s,s')$ is compatible with $\pi$,
        $s$ is non-goal state, and both $s$ and $s'$ are non dead-end states.
        Since $\pi$ solves $\P_\S$, the graph $\S_\pi$ is acyclic and $V(s,d)$ is well defined for non-dead-end states $s$.
\end{list}
    We show that $\sigma$ satisfies the formulas that make up $T$:
    \begin{enumerate}[1.]
        \item $\bigvee_{(s,s')} \Good{s,s'}$ clearly holds since if $s$ is a non-goal and non-dead-end state in $P_i$, the problem $P_i[s]$ belongs to $\P_\S$
        and thus is solvable by $\pi$. Then, there is at least one transition $(s,s')$ in $\S_\pi$ (i.e., compatible with $\pi$).
        \item Straightforward. There is $\delta$ such that $V^*(s)\leq V_\pi(s) \leq \delta V^*(s)$.
        \item Since $\S_\pi$ is acyclic, if $\Good{s,s'}$ holds, $(s,s')$ in $\S_\pi$ and $V_\pi(s')<V_\pi(s)$.
        \item Straightforward. If only one of $\{s,s'\}$ is goal, there is some feature $f{\in}\Phi$ such that $f(s){\neq} f(s')$ since $\Phi$ discriminates goals from non-goals.
        \item From definition, $\S_\pi$ has no transition $(s,s')$ where $s$ is non-goal and non-dead-end and $s'$ is a dead-end. Hence, $\sigma\vDash \neg\Good{s,s'}$.
        \item Let $(s,s')$ and $(t,t')$ be two transitions between non-dead-end states where $s$ and $t$ are both also non-goal states.
        If $\Good{s,s'}$ and $\neg\Good{t,t'}$ then both transitions must be \textbf{separated} by at least some feature $f\in\Phi$ since otherwise, if $(s,s')$ is compatible with $\pi$ (defined over $\Phi$),
        then $(t,t')$ would be also compatible with $\pi$ and thus $\Good{t,t'}$ would hold as well.
    \end{enumerate}
    Hence, $\sigma\vDash T$.

    \medskip
    For the converse direction, let $\sigma$ be a satisfying assignment for $T$.
    We must construct a subset $\Phi$ of features from $\F$ that discriminates goal from non-goal states in $P_1,\ldots,P_k$,
    and a policy $\pi=\pi_\Phi$ that solves $\P_\S$.
    Constructing $\Phi$ is easy: $\Phi=\{f\in\F:\sigma\vDash\Sel{f}\}$.
    For defining the policy, let us introduce the following idea.
    For a subset $\Phi$ of features and a subset $\T$ of transitions in $\S$,
    the policy $\pi_\T$ is the policy given by the rules $\Phi(s) \mapsto E_1 \,|\, \cdots \,|\, E_m$ where
    \begin{list}{--}{}
        \item $s$ is a ``source'' state in some transition $(s,s')$ in $\T$,
        \item $\Phi(s)$ is set of boolean conditions given by $\Phi$ on $s$;
        i.e., $\Phi(s)=\{ p : p(s){=}\text{true} \} \cup \{ \neg p : p(s){=}\text{false} \} \cup \{ n{>}0 : n(s){>}0 \} \cup \{ n{=}0 : n(s){=}0 \}$
        where $p$ (resp.\ $n$) is a boolean (resp.\ numerical) feature in $\Phi$,
        \item each $E_i$ captures the feature changes for some transition $(s',s'')$ in $\T$ such that $s\vDash \Phi(s)$; i.e.,
        $E_i$ is a \textbf{maximal set} of feature effects that is compatible with $(s',s'')$.
    \end{list}
    The policy $\pi$ is the policy $\pi_\T$ for $\T=\{(s,s') \in \S : \sigma\vDash\Good{s,s'} \}$.
    Observe that the policy $\pi$ is well defined since for two transitions $(s,s')$ and $(t,t')$
    such that $\Phi(s) = \Phi(t)$, the two rules associated with $\Phi(s)$ and $\Phi(t)$, respectively,
    are identical. This follows by formula (6) in the theory.
    By formula (4) in the theory, the features in $\Phi$ discriminate goal from non-goal states in $P_1,\ldots,P_k$.
    So, we only need to show that $\pi$ solves any problem in $\P_\S$.

    As before, let us construct the subgraph $\S_\pi$ of $\S$ spanned by $\pi$: the edge $(s,s')$ is in $\S_\pi$
    iff $s$ is a non-goal and non-dead-end state and $(s,s')$ is compatible with $\pi$; equivalently, $\sigma\vDash\Good{s,s'}$.
    Since $\S_\pi$ contains all states that are reachable in $P_1,\ldots,P_k$, a \textbf{necessary} and \textbf{sufficient} condition
    for $\pi$  to solve any problem in $\P_\S$ is that $\S_\pi$ is acyclic and each non-dead-end state in $\S_\pi$
    is connected to a goal state.
    The first property is a consequence of the assignments $V(s,d)$ to each non-dead-end state $s$ in $d$ since by formula (3),
    if $(s,s')$ belongs to $\S_\pi$ and $\sigma\vDash V(s,d)$, then $\sigma\vDash V(s',d')$ for some $0\leq d'<d$.
    For the second property, if $s$ is a non-goal and non-dead-end state in $\S_\pi$, then by formula (2),
    $\sigma\vDash V(s,d)$ for some $d$, and by formulas (1) and (3), $s$ is connected to a state $s'$ in $\S_\pi$ such that $\sigma\vDash V(s',d')$
    for some $0\leq d'<d$. The state $s'$ is not a dead end by formula (5). If $s'$ is not a goal state, repeating the
    argument we find that $s'$ is connected to a non-dead-end state $s''$ such that $\sigma\vDash V(s'',d'')$ with $0\leq d'' < d' < d$.
    This process is continued until a goal state $s^*$ connected to $s$ is found, for which $\sigma\vDash V(s^*,0)$.
    Therefore, $\pi$ solves any problem $P$ in $\P_\S$.
\end{proof}

\Omit{
\setcounter{definition}{5}
\begin{theorem}
    Let $\cal S$ be the state space associated with a set $P_1, \ldots, P_k$ of sample instances of a class of problems
    $\Q$ over a domain $D$, and let $\cal F$ be a pool of features.
    The theory $T({\cal S},{\cal F})$ is \textbf{satisfiable} iff
    there is a general policy $\pi_\Phi$
    over features $\Phi \subseteq \F$
    that solves $P_1$, \ldots, $P_k$
    and is able to discriminate all goals in these sample instances.
    \label{thm:sat}
\end{theorem}
\begin{proof}
\end{proof}

\begin{theorem}
    The policy $\pi_\Phi$ over the features $\Phi$ that  are determined by  a satisfying assignment of the theory $T$
    solves the sample problems $P_1, \ldots, P_k$ used to generate the state space $\cal S$, as well as the  problems
    $P_i[s]$ that are  like $P_i$ but with the initial state $s_0$ of $P_i$ set to each non-dead state $s$ reachable
    from $P$ from $s_0$.
    \label{thm:several-initial-states}
\end{theorem}
\begin{proof}
\end{proof}
}

\section{Feature Grammar}
The set $\F$ of candidate features is generated through a
standard description logics
grammar~\cite{dl-handbook}, similarly to \cite{bonet:aaai2019,frances-et-al-ijcai2019}.
Description logics build on the notions of \emph{concepts}, classes of objects that
have some property, and \emph{roles}, relations between these objects.
We here use the standard description logic $\mathcal{SOI}$ as a building block
for our features.

We start from a set of \emph{primitive concepts and roles} made up of all
unary and binary predicates that are used to define the PDDL model corresponding
to the generalized problem.\footnote{
This feature generation process implicitly restricts the domains that
\dtwol{} can tackle to those having predicates with arity $\leq 2$.
}
Following~\citeay{martin:generalized}, we also consider \emph{goal versions}
$p_g$ of each predicate $p$ in the PDDL model that is relevant for the goal.
These have fixed denotation in all
states of a particular problem instance, given by the goal formula.
To illustrate, a typical Blocksworld instance with a goal like  $on(x, y)$ results
in primitive concepts \emph{clear}, \emph{holding}, \emph{ontable}, and primitive roles
\emph{on} and \emph{on$_g$}.

In generalized domains where it makes sense to define a goal in terms
of a few \emph{goal parameters} (e.g., ``clear block $x$''), we take these
into account in the feature grammar below. Note however that this is mostly
to improve interpretability, and could be easily simulated without the need
for such parameters.

\subsection{Concept Language: Syntax and Semantics}

Assume that
$C$ and $C'$ stand for concepts, $R$ and $R'$ for roles,
and $\Delta$ stands for the \emph{universe} of a particular problem instance,
made up by all the objects appearing on it.
The set of all concepts and roles and their denotations in a given state $s$
is inductively defined as follows:

\begin{list}{--}{}
\item Any primitive concept $p$ is a concept with denotation $p^s = \setst{a}{s \models p(a)}$,
and primitive role $r$ is a role  with denotation $r^s = \setst{(a, b)}{s \models r(a, b)}$.

\item The \emph{universal concept} $\top$ and the \emph{bottom concept} $\bot$ are concepts
with denotations
$\top^s = \Delta$,
$\bot^s = \emptyset$.

\item The \emph{negation} $\neg C$,
the \emph{union} $C \sqcup C'$,
the \emph{intersection} $C \sqcap C'$ are concepts with denotations
$(\lnot C)^s = \Delta \setminus C^s$,
$(C \sqcup C')^s = C^s \cup {C'}^s$,
$(C\sqcap C')^s = C^s \cap {C'}^s$.

\item The \emph{existential restriction} $\exists R.C$ and the \emph{universal restriction} $\forall R.C$
are concepts with denotations
$(\exists R.C)^s = \{a \mid \exists b: (a,b) \in R^s \land b \in C^s\}$,
$(\forall R.C)^s = \{a \mid \forall b: (a,b) \in R^s \to b \in C^s\}$.

\item The \emph{role-value map} $R = R'$ is a concept
with denotation
$(R = R')^s =\{a \mid \forall b: (a,b) \in R^s \leftrightarrow (a,b) \in {R'}^s\}$.

\item If $a$ is a constant in the domain or a goal parameter,
the \emph{nominal} $\set{a}$ is a concept with denotation
$\set{a}^s = \set{a}$.

\item The \emph{inverse} role $R^{-1}$,
the \emph{composition} role $R \circ R'$ and
the (non-reflexive) \emph{transitive closure} role $R^+$
are roles with denotations
$(R^{-1})^s = \{(b, a) \mid (a,b) \in R^s\}$,
$(R \circ R')^s = \{(a, c) \mid \exists b: (a,b) \in R^s \land (b, c) \in {R'}^s\}$,
$(R^+)^s = \{(a_0, a_n) \mid \exists a_1, \dots, a_{n-1} \forall_{i} (a_{i-1}, a_i) \in R^s \}$.
\end{list}

We place some restrictions to the above grammar in order to reduce
the combinatorial explosion of possible concepts:
(1)~we do not generate concept unions,
(2)~we do not generate role compositions,
(3)~we only generate role-value maps $R = R'$ where $R'$ is the goal version of a $R$.
(4)~we only generate the inverse and transitive closure roles
$r^{-1}_p$, $r^{+}_p$, $(r^{-1}_p)^+$,
with $r_p$ being a primitive role.

\subsection{From concepts to features}

The \emph{complexity} of a concept or role is defined as the size of its syntax tree.
From the above-described infinite set of concepts and roles, we only consider the finite subset
$\G^k$ of those with complexity under a given bound $k$.
When generating  $\G^k$, redundant concepts and roles are pruned.
A concept or role is redundant when its denotation over all states in the training set
is the same as some previously generated concept or role.
From the domain model and $\G^k$, we generate the following features:
\begin{itemize}
  \item For each nullary primitive predicate $p$, a \emph{boolean}
    feature $b_p$ that is true in $s$ iff $p$ is true in $s$.
  %
  \item For each concept $C \in \G^k$,
        we generate a \emph{boolean} feature $|C|$, if $|C^s| \in \set{0,1}$
        for all states $s$ in the training set,
        and a \emph{numerical} feature $|C|$ otherwise.
        The value of boolean feature $|C|$ in $s$ is true iff $|C^s| = 1$;
        the value of numerical feature $|C|$ is $|C^s|$.
  \item Numerical features $\textit{Distance}(C_1,R{:}C,C_2)$ that represent the
    smallest $n$ such that there are objects $x_1, \ldots, x_n$ satisfying
    $C_1^s(x_1)$, $C_2^s(x_{n})$, and $(R{:}C)^s(x_i,x_{i+1})$ for $i=1,\ldots,n$.
    The denotation $(R{:}C)^s$ contains all pairs $(x,y)$ in $R^s$ such that $y\in C^s$.
    When no such $n$ exists, the feature evaluates to $m+1$, where $m$ is the number of
    objects in the particular problem instance.
\end{itemize}

The complexity $w(f)$ of a feature $f$ is set to the complexity of $C$ for features $|C|$,
to $1$ for features $b_p$,  and to the sum of the complexities of $C_1$, $R$, $C$, and $C_2$, for
features $\textit{Distance}(C_1,R{:}C,C_2)$.
Only features with complexity bounded by $k$ are generated.
For efficiency reasons we only generate features
$\textit{Distance}(C_1,R{:}C,C_2)$ where the denotation of concept $C_1$
in all states contains one single object.
All this feature generation procedure follows~\cite{bonet:aaai2019}, except
for the addition of goal predicates to the set of primitive concepts and roles.

\section{Generalized Policies}
We next describe in detail the generalized policies learned by \dtwol{}
on the reported example domains.\footnote{
The encodings of those domains below that are standard benchmarks from competitions
and literature can be obtained at \url{https://github.com/aibasel/downward-benchmarks}.
}
We also show that they generalize over the entire domain.

\paragraph{Reasoning about correctness.\footnote{
The following discussion relates to~\cite{seipp-et-al-ijcai2016}.
}}
We sketch a method to prove correctness of a policy over an entire
generalized planning domain $\Q$ in a domain-dependent manner.
Let $P$ be an instance of $\Q$.
We assume that $\Q$ implicitly defines what states of $P$
are valid; henceforth we are only concerned about valid states.
We say that a valid state of $P$ is \emph{solvable} if there
is a path from it to some goal in $P$,
and is \emph{alive} if it is solvable but not a goal.
We denote by $\mathcal{A}(P)$ the set of alive states of instance $P$.
In dead-end-free domains, where all states are solvable,
any state is either alive or a goal.

\begin{definition}[Complete \& Descending Policies]
We say that generalized policy $\pi_\Phi$ is \emph{complete over $P$} if for any
state $s \in \mathcal{A}(P)$, $\pi_\Phi$ is compatible with some transition
$(s, s')$.
We say that $\pi_{\Phi}$ is \emph{descending over $P$} if there is some
function $\gamma$ that maps states of $P$ to a totally ordered set $\mathcal{U}$
such that for any alive state $s \in \mathcal{A}(P)$
and $\pi_\Phi$-compatible transition $(s, s')$, we have that $\gamma(s') < \gamma(s)$.
\end{definition}

\begin{theorem}
Let $\pi_{\Phi}$ be a policy that is complete and descending for $P$.
Then, $\pi_{\Phi}$ solves $P$.
\end{theorem}

\begin{proof}
Because $\pi_{\Phi}$ is descending, no state trajectory compatible with it
can feature the same state more than once.
Since the set $S(P)$ of states of $P$ is finite,
there is a finite number of trajectories compatible with $\pi_{\Phi}$,
all of which have length bounded by $\abs{S(P)}$.
Let $\tau$ be one \emph{maximal} such trajectory, i.e., a trajectory
$\tau = s_0, \ldots, s_n$ such that $P$ allows no $\pi_{\Phi}$-compatible transition
$(s_n, s)$. Because $\pi_{\Phi}$ is complete, $s_n \not\in \mathcal{A}(P)$,
so it must be a goal.
\end{proof}

A way to show that $\pi_{\Phi}$ is descending is by providing
a fixed-length tuple $\tup{f_1, \ldots, f_n}$ of state features
$f_i: S(P) \mapsto \mathbb{N}$.
Boolean features can have their truth values cast to $0$ (false) or 1 (true).
%
If for every transition $(s, s')$ compatible with $\pi_{\Phi}$,
$\tup{f_1(s'), \ldots, f_n(s')} < \tup{f_1(s), \ldots, f_n(s)}$,
where $<$ is the \textbf{lexicographic} order over tuples,
then $\pi_{\Phi}$ is descending.
When this is the  case, we say that $\pi_{\Phi}$
\emph{descends over $\tup{f_1, \ldots, f_n}$}.

\subsection{Policy for $\Q_{clear}$}

The set of features $\Phi$ learned by \dtwol{} contains:
\begin{itemize}
\item $c \equiv \abs{clear \sqcap \{x\} }$: whether block $x$ is clear.
\item $H \equiv \abs{holding}$: whether the gripper is holding some block.
\item $n \equiv \abs{\exists on^+ . \{x\}}$: number of blocks above $x$.
\end{itemize}

\vspace*{6pt} \noindent The learned policy $\pi_\Phi$ is:
\begin{alignat*}{1}
&r_1:\, \set{\neg c, H, n=0}        \ \mapsto\ \set{c, \neg H}\,, \\
&r_2:\, \set{\neg c, \neg H, n>0}   \ \mapsto\ \set{c?, H, n\mminus}\,, \\
&r_3:\, \set{\neg c, H, n>0}        \ \mapsto\ \set{\neg H}\,.
\end{alignat*}

There are no dead-ends in $\Q_{clear}$, and $c$ is true only in goal states.
A particularity of the 4-op encoding used here is that when a block is being
held, it is not considered clear. Hence, $n=0$ does not imply the goal.

Let us show that $\pi_\Phi$ is complete over any problem instance $P$.
Let $s \in \mathcal{A}(P)$. If the gripper is holding some block in $s$, then
the transition where it puts the block on the table is compatible with $\pi_\Phi$
(rules $r_1$, $r_3$), regardless of whether the block is $x$ or not.
If the gripper is empty, there must be at least one block above $x$,
otherwise $s$ would be a goal. The transition where the gripper picks one such block
is compatible with $\pi_\Phi$ ($r_2$).

Now, let us show that $\pi_\Phi$ descends over feature tuple $\tup{n, H}$.
Rules $r_1$ and $r_3$ do not affect $n$ and make $H$ false, so any transition $(s, s')$ compatible
with them makes the valuation of $\tup{n, H}$ decrease.
Rule $r_2$ always decreases $n$, so compatible transitions decrease $\tup{n, H}$.
Since $\pi_\Phi$ is complete and descending, it solves $P$.
\qed

\subsection{Policy for $\Q_{on}$}
The set of features $\Phi$ learned by \dtwol{} contains:
\begin{itemize}
\item $e \equiv \abs{handempty}$: whether the gripper is empty,
\item $c \equiv \abs{clear}$: number of clear objects,
\item $c(x) \equiv \abs{clear \sqcap \{x\} }$: whether block $x$ is clear,
\item $on(y) \equiv \abs{\exists on. \{y\}}$: whether some block is on $y$.
\item $ok \equiv \abs{\{x\} \sqcap \exists on.\{y\}}$: whether $x$ is on $y$,
\end{itemize}



\vspace*{6pt} \noindent The learned policy $\pi_\Phi$ is:
\begin{alignat*}{1}
&r_1:\, \{e, c(x), \neg on(y)\} \mapsto \set{\neg e, \neg c(x), c\mminus} \psep \set{\neg e, \neg c(x)}, \\
&r_2:\, \{e, c(x), on(y)\} \mapsto \set{\neg e} \psep \set{\neg e, \neg on(y)} \psep \set{\neg e, \neg c(x)}, \\
&r_3:\, \{e, \neg c(x), \neg on(y)\} \mapsto \set{\neg e} \psep \set{\neg e, c(x)}, \\
&r_4:\, \{e, \neg c(x), on(y)\} \mapsto \set{\neg e} \psep \set{\neg e, \neg on(y)} \psep \set{\neg e, c(x)}, \\
&r_5:\, \{\neg e, c(x)\}   \mapsto \set{e, c\pplus}, \\
&r_6:\, \{\neg e, \neg c(x), on(y)\}   \mapsto \set{e, c(x), c\pplus} \psep \set{e, c\pplus}, \\
&r_7: \set{\neg e, \neg c(x), \neg on(y)} \mapsto \set{e, c(x), ok, on(y)} | \set{e, c\pplus}.
\end{alignat*}

There are no dead-ends in $\Q_{on}$, and in all alive states, $c>0$ and $\neg ok$.
These two conditions have been omitted in the body of all 7 rules above for readability.

Let us show that $\pi_\Phi$ is complete over any problem instance $P$.
Note that the conditions in the rule bodies nicely partition $\mathcal{A}(P)$.
Let $s \in \mathcal{A}(P)$ be an alive state.
If the gripper is holding some block in $s$, then
putting it on the table is always possible and is compatible with rules $r_5$--$r_7$,
except when the held block is $x$ and $y$ has nothing above. In that case, putting
$x$ on $y$ is possible and compatible with $r_7$.
If the gripper is empty, consider whether $x$ and $y$ are clear.
If both are clear, then picking up $x$ is always possible, and is compatible with $r_1$.
Otherwise, there must be at least one tower of blocks, and
picking up some block from such a tower is always possible,
and is compatible with $r_2$--$r_4$.

Now, let us show that $\pi_\Phi$ descends over feature tuple $\tup{al, ready', t', e}$,
where
$al$ is $1$ in any alive state, and $0$ otherwise;
$ready'$ is $0$ if $holding(x)$ and $clear(y)$, and 1 otherwise;
$t'$ is the number of blocks not on the table,
and $e$ is as defined above.
Rules $r_1$--$r_4$ are compatible only with transitions where the gripper
is initially empty; none affects $al$, and all decrease $e$.
All their effects are compatible only with pick-ups from a tower (otherwise $c\mminus$),
hence do not affect $t'$, except for picking up $x$ when $y$ is clear,
(first effect of $r_1$), which increases $t'$ but makes $ready'$ decrease.
Rules $r_5$--$r_6$ are compatible only with putting a held block on the table,
decreasing $t'$, and do not affect $al$ or $ready'$.
A similar reasoning applies to rule $r_7$, except when the held block is $x$
and can be put on $y$, in which case $al$ decreases.

Since $\pi_\Phi$ is complete and descending, it solves $P$.
\qed

\subsection{Policy for $\Q_{grip}$}
The set of features $\Phi$ learned by \dtwol{} contains:
\begin{itemize}
 \item $r_B \equiv \abs{\exists at_g . \mli{at-robby}}$: whether the robot is at $B$.
 \item $c \equiv \abs{\exists carry . \top}$: number of balls carried by the robot.
 \item $b \equiv \abs{\neg (at_g = at)}$: number of balls not in room $B$.
\end{itemize}

\vspace*{6pt} \noindent The learned policy $\pi_\Phi$ is:
\begin{alignat*}{1}
&r_1:\, \{\neg r_B, c=0, b>0\}\ \mapsto\ \{c\pplus\}\,, \\
&r_2:\, \{r_B, c=0, b>0\}\ \mapsto\ \{\neg r_B\}\,, \\
&r_3:\, \{r_B, c>0, b>0\}\ \mapsto\ \{c\mminus, b\mminus\}\,, \\
&r_4:\, \{\neg r_B, c>0, b>0\}\ \mapsto\ \{r_B\}\,.
\end{alignat*}

There are no dead-ends in Gripper, and $b>0$ in any alive state of an instance $P$.
Let us show that $\pi_\Phi$ is complete over any problem instance $P$.
Let $s$ be an alive state.
If the robot is in room $A$ carrying some ball, the transition where it moves
to $B$ is compatible with $\pi_\Phi$ ($r_4$). If it is carrying no ball, then
there must be some ball in $A$ ($s$ is not a goal); picking it
is compatible with $r_1$.
Now, if the robot is in room $B$ carrying some ball, the transition where it drops it
is compatible with $r_3$; if it carries no ball, the transition where it moves to room
$A$ is compatible with $r_2$.

Policy $\pi_\Phi$ descends over tuple
$\tup{b_A, b_{RA}, b_{RB}, r_B}$, where
$b_A$ counts the number of balls in room $A$,
$b_{Rx}$ the number of balls held by the robot while in room $x$,
and $r_B$ is as defined above.
This is because
rule $r_1$ decreases $b_{A}$;
rule $r_2$ decreases $r_B$ without affecting the other features;
rule $r_3$ decreases $b_{RB}$, and
rule $r_4$ increases $b_{RB}$ but decreases $b_{RA}$.
Since $\pi_\Phi$ is complete and descending, it solves $P$.
\qed

\subsection{Policy for $\Q_{rew}$}
The set of features $\Phi = \{u, d\}$ contains features:
\begin{itemize}
\item $u \equiv \abs{reward}$: number of unpicked rewards.
\item $d \equiv Distance(at, adjacent:unblocked, reward)$:
Distance between the agent and the closest cell with some unpicked reward
along a path of unblocked cells.
\end{itemize}

\vspace*{6pt} \noindent The learned policy $\pi_\Phi$ is:
\begin{alignat*}{1}
  &r_1:\, \{r>0, d=0\}\ \mapsto\ \{ r\mminus, d\pplus\}\,, \\
  &r_2:\, \{r>0, d>0\}\ \mapsto\ \{ d\mminus\}\,.
\end{alignat*}

There are no dead-ends in $\Q_{rew}$.
Let us show that $\pi_\Phi$ is complete over any problem instance $P$.
Let $s$ be an alive state, hence $r>0$.
If the agent is in a cell with reward, picking the reward is always compatible with
rule $r_1$, since there can be at most one reward item per cell.
If there is no reward in the cell, then there must be a reward at some finite distance
(otherwise $s$ would either be a goal or unsolvable), and moving closer to the closest
reward is always possible and compatible with $r_2$.

It is straight-forward to see from the rule effects that $\pi_\Phi$ descends over tuple $\tup{r, d}$,
hence it solves $P$.
\qed

\subsection{Policy for $\Q_{deliv}$}
%
%
%
%
%

The set of features $\Phi$ learned by \dtwol{} contains:
\begin{itemize}
    \item $e \equiv \abs{empty}$: whether the truck is empty.
    \item $u \equiv \abs{\neg (at_g = at)}$: number of undelivered packages.
    \item $du \equiv \abs{Distance(C_{t}, adjacent, C_{cup})}$: distance between truck and closest undelivered package.
    \item $dt \equiv \abs{Distance(C_t, adjacent, \exists at_g^{-1}.\top)}$: distance between truck and target location.
\end{itemize}

For readability we have used  $C_{t} \equiv \exists at^{-1}.truck$ to stand for the concept denoting the
location of the truck, and
$C_{cup} \equiv (\forall at_g^{-1}.\bot) \sqcap (\exists at^{-1}.package)$
for the concept denoting the set of cells with some undelivered package on them.

\vspace*{6pt} \noindent The learned policy $\pi_\Phi$ is:
\begin{alignat*}{1}
&r_1:\, \set{\neg e, du>0, dt=0, u>0}      \mapsto\ \set{e, u\mminus}\, \\
&r_2:\, \set{\neg e, du=0, dt>0, u>0}      \mapsto\ \set{du \pplus, dt\mminus} \psep \set{dt\mminus} \, \\
&r_3:\, \set{\neg e, du>0, dt>0, u>0}      \mapsto\ \set{du ?, dt\mminus} \, \\
&r_4:\, \set{e, du=0, dt>0, u>0}      \mapsto\ \set{\neg e, du \pplus} \psep \set{\neg e}\, \\
&r_5:\, \set{e, du>0, dt=0, u>0}      \mapsto\ \set{du \mminus, dt \pplus} \, \\
&r_6:\, \set{e, du>0, dt>0, u>0}      \mapsto\ \set{du \mminus, dt \mminus} \psep \set{du \mminus, dt \pplus}\,.
\end{alignat*}

There are no dead-ends in Delivery.
Let us show that $\pi_\Phi$ is complete over any problem instance $P$.
Let $s$ be an alive state of $P$.
We know that $u>0$ in $s$, otherwise the state would be a goal.
Assume first that the truck is carrying some package.
If it is on the target location, $r_1$ is compatible with dropping the package,
which is always possible in the domain. If it is not, $r_2$--$r_3$ are compatible
with moving towards the target location, which is always possible, and cover all
possibilities regarding the distance to some other undelivered package.
Assume now the truck is empty. If it is on the same location as an undelivered package,
$r_4$ is compatible with picking the package, which is always possible.
If it is not, moving towards the closest undelivered package is possible and
compatible with rules $r_5$--$r_6$, which cover all possibilities regarding
the distance to the target.\footnote{Note that in a grid, moving along the four
cardinal directions always changes the (Manhattan) distance to any third, fixed cell.}

Let $\tup{u, e, du_{e}, dt_{\neg e}}$ be a feature tuple where
$u$ and $e$ are as defined above,
$du_{e}$ is equal to $du$ when the truck is empty, and to $0$ otherwise, and
$dt_{\neg e}$ is equal to $dt$ when the truck is carrying a package, and $0$ otherwise.
Syntactic inspection of the policy $\pi_\Phi$ shows that for any transition compatible with it,
the valuation of the feature tuple lexicographically decreases.
In rules $r_1$ and $r_4$, this is because of features $u$ and $e$;
in $r_2$--$r_3$, $dt_{\neg e}$ always decreases, and $u, e, du_{e}$
do not change their value;
in $r_5$--$r_6$, $du_{e}$ always decreases, whereas $u, e$ do not change their value.
Hence, $\pi_\Phi$ descends over the given feature tuple, and thus
solves $\Q_{deliv}$.
\qed

\subsection{Policy for $\Q_{visit}$}

The set of features $\Phi$ learned by \dtwol{} contains:
\begin{itemize}
\item $u \equiv \abs{\neg visited}$: number of unvisited objects.
\item $d \equiv Distance(\mli{at-robot}, connected, \neg visited)$: Distance between the robot and the closest unvisited cell.
\end{itemize}

\vspace*{6pt} \noindent The learned policy $\pi_\Phi$ is:
\begin{alignat*}{1}
&r_1:\, \{u>0, d>0\}\ \mapsto\ \{d \mminus \} \psep \{u \mminus, d \pplus \} \psep \{u \mminus \}
\end{alignat*}

There are no dead-ends in Visitall.
Let us show that $\pi_\Phi$ is complete over any problem instance $P$.
Let $s$ be an alive state,
hence $u>0$
and $d>0$ (as soon as the robot steps into an unvisited cell, it becomes visited).
If $d>1$, the robot can always move closer to an unvisited cell, which is compatible
with the first effect.
If $d=1$, the robot can always move into one of the distance-1 unvisited cells $x$,
which is compatible with one of the effects 2 and 3, depending on whether there is
some other unvisited cell adjacent to $x$ or not.

It is straight-forward to see from the rule effects that $\pi_\Phi$ descends over
tuple $\tup{u, d}$, and hence solves $P$.
\qed

\subsection{Policy for $\Q_{span}$}
The set of features $\Phi$ learned by \dtwol{} contains:
\begin{itemize}
\item $n \equiv \abs{tightened_g \sqcap \neg tightened}$: number of untightened nuts,
\item $h \equiv \abs{\exists at. \top}$: number of objects not held by the agent,
\item $e \equiv \abs{\exists at. (\forall at^{-1}.man)}$: whether the agent location is empty, i.e., there is no spanner or nut in it.
\end{itemize}

\vspace*{6pt} \noindent The learned policy $\pi_\Phi$ is:
\begin{alignat*}{1}
&r_1:\, \set{n>0,  h>0, e}      \ \mapsto\ \set{e?}\,, \\
&r_2:\, \set{n>0,  h>0, \neg e} \ \mapsto\ \set{h\mminus, e?} \psep \set{n\mminus}\,.
\end{alignat*}

It is clear that $n>0$ for all alive states in $\mathcal{A}$,
and $h>0$ in all states, as e.g.\ the nuts cannot be held by the agent,
and will always be \emph{at} some location.

Let $s$ be an alive state of an arbitrary problem instance $P$.
If the agent is on the gate, there must be some nut to be tightened
($s$ is not a goal). Since $s$ is solvable, the agent must be carrying
a usable spanner,
and tightening the nut is always an option, compatible with $r_2$, second effect.
Otherwise the agent is not on the gate. If there is some spanner to be picked up,
doing so is always possible, and compatible with $r_2$, first effect.
If instead the location is empty, moving right is always an option compatible
with rule $r_1$, and will affect $e$ or not depending on whether the
next location is empty or not.

Spanner is the only domain with dead-ends that we consider,
so we cannot show correctness simply by showing descendingness of $\pi_\Phi$.
However, the following argument gives an intuition of why the policy is correct:
%
$\pi_\Phi$ descends over $\tup{ms, n, d}$,
where $ms$ is the number of missed spanners, i.e.,
spanners that have been left behind and cannot be picked up anymore,
$n$ is as defined above, and $d$ is the distance between the agent and the gate.
Since $\pi_\Phi$ is not compatible with any transition that increases
$ms$, it avoids dead-ends.
\qed

\subsection{Policy for $\Q_{micon}$}
The set of features $\Phi$ learned by \dtwol{} contains:\footnote{
Note that we use the standard Miconic encoding, but after fixing a minor bug that
would allow passengers to magically appear in the origin floor at any time after they
have been transported to their destination floor.
}
\begin{itemize}
\item $b \equiv \abs{boarded}$: number of passengers onboard the lift,
\item $w \equiv \abs{\exists \mli{waiting-at} . \top}$: number of passengers waiting to board,
\item $rb \equiv \abs{\forall \mli{waiting-at} . \mli{lift-at}}$:
number of elements that either are not waiting to board, or are waiting
and \emph{ready to board}, i.e., on the same floor as the lift.

\item $rl \equiv \abs{boarded \sqcap \exists destin. \mli{lift-at}}$:
number of passengers \emph{ready to leave}, i.e.\ boarded with the lift on their destination floor.
\end{itemize}


\vspace*{6pt} \noindent The learned policy $\pi_\Phi$ is:
\begin{alignat*}{1}
&r_1:\, \{b=0,  w>0, rb>0, rl=0\}\ \mapsto\ \set{rb \pplus} \psep \set{w \mminus, b \pplus}     \\
&r_2:\, \{b>0,  w=0, rb>0, rl=0\}\ \mapsto\ \set{rl \pplus}     \\
&r_3:\, \{b>0,  w=0, rb>0, rl>0\}\ \mapsto\ \set{b \mminus, rl \mminus}     \\
&r_4:\, \{b>0,  w>0, rb>0, rl=0\}\ \mapsto\ \set{rl \pplus, rb ?} \psep \set{w \mminus, b \pplus} \\
&r_5:\, \{b>0,  w>0, rb>0, rl>0\}\ \mapsto\ \set{b \mminus, rl \mminus} \psep \set{w \mminus, b \pplus}.
\end{alignat*}

There are no dead-ends in Miconic, and in all states, $rb>0$ since (counterintuitively)
there is always some element, e.g., a floor, that is not waiting to board.
Note that the rule conditions partition the entire space of alive states, since
$b=0$ trivially implies $rl=0$, and it also implies that $w>0$, as if no passenger is waiting nor boarded,
it must be that she has been delivered to her destination (Miconic
does not allow passengers to leave on a floor other than their destination).

Let us show that $\pi_\Phi$ is complete over any problem instance $P$.
Let $s$ be an alive state.
If the lift is empty, there must be some passenger waiting.
If she is on the same floor, boarding her is compatible with $r_1$; otherwise,
moving to her floor is also compatible with $r_1$.
Assume now that someone is boarded on the lift.
If the lift is on her destination floor, leaving the lift is compatible with $r_3$, $r_5$.
If the lift is not on the floor of any boarded passenger, moving to the floor
of some boarded passenger is compatible with $r_2$, $r_4$.

Now, let us show that $\pi_\Phi$ descends over feature tuple
$\tup{w, b, m-rl, m-rb}$, where
$m$ is the total number of objects in $P$, and $w$, $rl$, $rb$ and $b$ are
as defined above.
By going through each of the rule effects, it is straight-forward to see that
allowed transitions always decrease the tuple valuation.
Since $\pi_\Phi$ is complete and descending, it solves $P$.
\qed

\subsection{Policy for $\Q_{bw}$}
In $\Q_{bw}$
we use a different but standard PDDL encoding where blocks are moved atomically
from one location to another (no gripper).
We say that a block $b$ is \emph{well-placed} if it is on its target location
(block or table), and so are all blocks below, otherwise $b$ is \emph{misplaced}.
Note that a goal can always be reached by moving misplaced blocks only.
The set of features $\Phi$ learned by \dtwol{} contains:
%

\begin{itemize}
\item $c \equiv \abs{clear}$: number of clear objects,
\item $t' \equiv \abs{\neg(on_g = on)}$:  number of objects that are not \emph{on} their final target,
\item $bwp \equiv \abs{\forall on^+. on_g = on}$: number of objects s.t.\ all objects below are well-placed,
i.e., in its goal configuration.
\end{itemize}

\vspace*{6pt} \noindent The learned policy $\pi_\Phi$ has one single rule $r_1$:
\begin{align*}
\set{c>0, t'>0, bwp>0}\ \mapsto\
&\set{c \pplus}                                  \psep
\set{c \pplus,  t' ?, bwp \pplus} \psep \\
&\set{c \pplus,  t' \mminus} \psep
\set{c \mminus, t' \mminus}.
\end{align*}

There are no dead-ends in $\Q_{bw}$.
All features in $\Phi$ are strictly positive on alive states
($c>0$ because there is always at least one tower,
$t'>0$ because otherwise the state is a goal,
$bwp>0$ because the table is always well-placed and below any block).
In this atomic-move encoding, $c\pplus$ iff some block goes from being
on another block to being on the table, and $c\mminus$ iff the opposite occurs.

Let us first prove that $\pi_\Phi$ is complete over any problem instance $P$.
Let $s$ be an alive state in $P$. Since $s$ is not a goal, there must be some
misplaced block. We make a distinction based on whether all
misplaced blocks are on the table or not.
If they are, it is easy to prove that there must be one of them, call it $b$,
such that its target location is clear and well-placed.
Moving $b$ onto its target location is then a possibility,
compatible with the last effect of $r_1$.

If, on the contrary, some misplaced block is not on the table, then it is
easy to see that there must be a misplaced block $b$ that is clear, since the
condition of being misplaced ``propagates'' upwards any tower of blocks.
Moving $b$ to the table is always possible; we only need to
show that it is always compatible with some rule effect.
If there is some misplaced block below $b$, then the move is compatible
with effect 2. Otherwise, it must be that $b$ is not on its target location.
If its target location is the table, moving to the table is
compatible with effect 3; otherwise, it is compatible with effect 1.

Additionally, $\pi_\Phi$ descends over tuple $\tup{wp', t''}$, where
$wp'$ is the number of blocks that are not well-placed, and
$t''$ the number of blocks that are not on the table.
Since $\pi_\Phi$ is complete and descending, it solves $P$.
Indeed,
$\pi_\Phi$ implements
the well-known policy that moves ill-placed blocks to the table,
then builds the target towers bottom-up.
\qed

\end{document}